\documentclass{article}

\usepackage{microtype}
\usepackage{graphicx}
\usepackage{booktabs} 

\usepackage{hyperref}

\usepackage[accepted]{icml2020}

\usepackage{amsmath,amssymb,amsfonts}
\usepackage{algorithmic}
\usepackage{enumitem}
\usepackage{graphicx}
\usepackage{textcomp}
\usepackage{xcolor}

\usepackage{amsthm}
\usepackage{booktabs}
\usepackage{dsfont}
\usepackage{multirow}
\usepackage{subcaption}
\usepackage{thmtools,thm-restate}
\usepackage{url}

\usepackage{csquotes}
\MakeOuterQuote{"}

\newtheorem{definition}{Definition}
\newtheorem{proposition}{Proposition}

\newtheorem{theorem}{Theorem}
\newtheorem{corollary}{Corollary}

\theoremstyle{definition}
\newtheorem*{remark}{Remark}

\icmltitlerunning{Bayesian Differential Privacy for Machine Learning}

\begin{document}

\twocolumn[
\icmltitle{Bayesian Differential Privacy for Machine Learning}



\icmlsetsymbol{equal}{*}

\begin{icmlauthorlist}
\icmlauthor{Aleksei Triastcyn}{epfl}
\icmlauthor{Boi Faltings}{epfl}
\end{icmlauthorlist}

\icmlaffiliation{epfl}{Artificial Intelligence Lab, Ecole Polytechnique F\'ed\'erale de Lausanne (EPFL), Lausanne, Switzerland}

\icmlcorrespondingauthor{Aleksei Triastcyn}{aleksei.triastcyn@epfl.ch}

\icmlkeywords{Machine Learning, Deep Learning, Differential Privacy, Privacy, Privacy Accounting}

\vskip 0.3in
]



\printAffiliationsAndNotice{}  

\begin{abstract}
Traditional differential privacy is independent of the data distribution. However, this is not well-matched with the modern machine learning context, where models are trained on specific data. As a result, achieving meaningful privacy guarantees in ML often excessively reduces accuracy. We propose \emph{Bayesian differential privacy (BDP)}, which takes into account the data distribution to provide more practical privacy guarantees. We also derive a general privacy accounting method under BDP, building upon the well-known moments accountant. Our experiments demonstrate that in-distribution samples in classic machine learning datasets, such as MNIST and CIFAR-10, enjoy significantly stronger privacy guarantees than postulated by DP, while models maintain high classification accuracy.
\end{abstract}

\section{Introduction}
\label{sec:introduction}
Machine learning (ML) and data analytics offer vast opportunities for companies, governments and individuals to take advantage of the accumulated data. However, their ability to capture fine levels of detail can compromise privacy of data providers. Recent research~\cite{fredrikson2015model, shokri2017membership, hitaj2017deep} suggests that even in a black-box setting it is possible to infer information about individual records in the training set.

Numerous solutions have been proposed to address this problem, varying in the extent of data protection and how it is achieved. In this work, we consider a notion that is viewed by many researchers as the gold standard -- \emph{differential privacy (DP)}~\cite{dwork2006}. Initially, DP algorithms focused on sanitising simple statistics, but in recent years, it made its way to machine learning~\cite{abadi2016deep,papernot2016semi,papernot2018scalable,mcmahan2018learning}.

Despite notable advances, differentially private ML still suffers from two major problems: (a) utility loss due to excessive noise added during training and (b) difficulty in interpreting the privacy parameters $\varepsilon$ and $\delta$. In many cases where the first problem appears to be solved, it is actually being hidden by the second. We design a motivational example in Section~\ref{sec:motivation} that illustrates how a seemingly strong privacy guarantee allows for the attacker accuracy to be as high as $99\%$. Although this guarantee is very pessimistic and holds against a powerful adversary with any auxiliary information, it can hardly be viewed as a reassurance to a user. Moreover, it provides only the worst-case bound, leaving users to wonder how far is it from a typical case.

In this paper, we focus on practicality of privacy guarantees and propose a variation of DP that provides more meaningful guarantees for \emph{typical} scenarios on top of the global DP guarantee. We name it \emph{Bayesian differential privacy (BDP)}.

The key to our privacy notion is the definition of \emph{typical} scenarios. We observe that machine learning models are designed and tuned for a particular data distribution (for example, an MRI dataset is very unlikely to contain a picture of a car). Moreover, such prior distribution of data is often already available to the attacker. Thus, we consider a scenario \emph{typical} when all sensitive data is drawn from the same distribution. While the traditional differential privacy treats all data as equally likely and hides differences by large amounts of noise, Bayesian DP calibrates noise to the data distribution and provides much tighter expected guarantees.

As the data distribution is usually \emph{unknown}, BDP estimates the necessary statistics from data as shown in the following sections. Furthermore, since typical scenarios are determined by data, the participants of the dataset are covered by the BDP guarantee with high probability.

To accompany the notion of Bayesian DP (Section~\ref{sec:bdp_definition}), we provide its theoretical analysis and the privacy accounting framework (Section~\ref{sec:accounting}). The latter considers the privacy loss random variable and employs principled tools from probability theory to find concentration bounds on it. It provides a clean derivation of privacy accounting in general (Sections~\ref{sec:accounting} and~\ref{sec:privacy_cost_estimator}), as well as in the special case of subsampled Gaussian noise mechanism. Moreover, we show that it is a generalisation of the well-known moments accountant (MA)~\cite{abadi2016deep} (Section~\ref{sec:relation_to_ma_rdp}).

Since our privacy accounting relies on data distribution samples, a natural concern is that the data not present in the dataset are not taken into account, and thus, are not protected. However, our finite sample estimator is specifically designed to address this issue (see Section~\ref{sec:privacy_cost_estimator}).

Our contributions in this paper are the following:
\begin{itemize}[nolistsep,noitemsep]
\item we propose a variation of DP, called Bayesian differential privacy, that allows to provide more practical privacy guarantees in a wide range of scenarios;
\item we derive a clean, principled privacy accounting method that generalises the moments accountant;
\item we experimentally demonstrate advantages of our method (Section~\ref{sec:evaluation}), including the state-of-the-art privacy bounds in deep learning (Section~\ref{sec:deep_learning}).
\end{itemize}

\section{Related Work}
\label{sec:related_work}
With machine learning applications becoming more and more ubiquitous, vulnerabilities and attacks on ML models get discovered, raising the need for matching defences. These attacks can be based on both passive adversaries, such as model inversion~\cite{fredrikson2015model} and membership inference~\cite{shokri2017membership}, and active adversaries (for example,~\cite{hitaj2017deep}).

One of the strongest privacy standards that can be employed to protect ML models from these and other attacks is differential privacy~\cite{dwork2006,dwork2006calibrating}. Pure $\varepsilon$-DP is hard to achieve in many realistic learning settings, and therefore, a notion of approximate $(\varepsilon, \delta)$-DP is used across-the-board in machine learning. It is typically accomplished by applying the Gaussian noise mechanism~\cite{dwork2014algorithmic} during the gradient descent update~\cite{abadi2016deep}. Privacy accounting, i.e. computing the privacy guarantee throughout multiple iterations of the algorithm, is typically done by the \emph{moments accountant (MA)}~\cite{abadi2016deep}. In Section~\ref{sec:relation_to_ma_rdp}, we discuss the link between MA and our accounting method, as well as connection to a closely related notion of R\'enyi DP~\cite{mironov2017renyi}. Similarly, a link can be established to concentrated DP definitions~\cite{dwork2016concentrated,bun2016concentrated}.

A number of previous relaxations considered a similar idea of limiting the scope of protected data or using the data generating distribution, either through imposing a set of data evolution scenarios~\cite{kifer2014pufferfish}, policies~\cite{he2014blowfish}, distributions~\cite{blum2013learning,bhaskar2011noiseless}, or families of distributions~\cite{bassily2013coupled,bassily2016typical}. Some of these definitions (e.g.~\cite{blum2013learning}) may require more noise, because they are stronger than DP in the sense that datasets can differ in more than one data point. This is not the case with our definition: like DP, it considers adjacent datasets \emph{differing in a single data point}. The major problem of such definitions, however, is that in real-world scenarios it is not feasible to exactly define distributions or families of distributions that generate data. And even if this problem is solved by restricting the query functions to enable the usage of the central limit theorem (e.g.~\cite{bhaskar2011noiseless,duan2009privacy}), these guarantees would only hold asymptotically and may require prohibitively large batch sizes. While Bayesian DP can be seen as a special case of some of the above definitions, the crucial difference with the prior work is that our additional assumptions allow the Bayesian accounting (Sections~\ref{sec:accounting},~\ref{sec:privacy_cost_estimator}) to provide guarantees w.h.p. with finite number of samples from data distributions, and hence, enable a broad range of real-world applications.

Finally, there are other approaches that use the data distribution information in one way or another, and coincidentally share the same~\cite{yang2015bayesian} or similar~\cite{leung2012bayesian} names. Yet, similarly to the methods discussed above, their assumptions (e.g. bounds on the minimum probability of a data point) and implementation requirements (e.g. potentially constructing correlation matrices for millions of data samples) make practical applications very difficult. Perhaps, the most similar to our approach is random differential privacy~\cite{hall2011random}. The main difference is that \citet{hall2011random} consider the probability space over all data points, while we only consider the space over a single differing example. As a result, our guarantees are more practical to compute for large, realistic ML datasets. Furthermore, \citet{hall2011random} only propose a basic composition theorem, which is not tight enough for accounting in iterative methods, and to the best of our knowledge, there are no proofs for other crucial properties, such as post-processing and group privacy. 

\section{Motivation}
\label{sec:motivation}
Before we proceed, we find it important to motivate research on alternative privacy definitions, as opposed to fully concentrating on new mechanisms for DP. On the one hand, there is always a combination of data and a desired statistic that would yield large privacy loss in DP paradigm, regardless of the mechanism. In other words, there can always be data outliers that are difficult to hide without a large drop in accuracy. On the other hand, we cannot realistically expect companies to sacrifice model quality in favour of privacy. As a result, we get models with impractical worst-case guarantees (as we demonstrate below) without any indication of what is the privacy guarantee for the majority of users.

Consider the following example. The datasets $D, D'$ consist of income values for residents of a small town. There is one individual $x'$ whose income is orders of magnitude higher than the rest, and whose residency in the town is what the attacker wishes to infer. The attacker observes the mean income $w$ sanitised by a differentially private mechanism with $\varepsilon=\varepsilon_0$ (we consider the stronger, pure DP for simplicity). What we are interested in is the change in the posterior distribution of the attacker after they see the private model compared to prior~\cite{mironov2017renyi,bun2017teaser}.  If the individual is not present in the dataset, the probability of $w$ being above a certain threshold is extremely small. On the contrary, if $x'$ is present, this probability is higher (say it is equal to $r$). The attacker computes the likelihood of the observed value under each of the two assumptions, the corresponding posteriors given a flat prior, and applies a Bayes optimal classifier. The attacker then concludes that the individual is present in the dataset and is a resident.

By the above expression, $r$ can only be $e^{\varepsilon_0}$ times larger than the respective probability without $x'$. However, if the $r e^{-\varepsilon_0}$ is small enough, then the probability $P(A)$ of the attacker's guess being correct is as high as $r / (r + re^{-\varepsilon_0})$, or
\begin{align}
\label{eq:motivation}
	P(A) = \frac{1}{1 + e^{-\varepsilon}}.
\end{align}

To put it in perspective, for a DP algorithm with $\varepsilon = 2$, the upper bound on the accuracy of this attack is as high as $88\%$. For $\varepsilon = 5$, it is $99.33\%$. For $\varepsilon = 10$, $99.995\%$. Importantly, these values of $\varepsilon$ are very common in DP ML literature~\cite{shokri2015privacy,abadi2016deep,papernot2018scalable}, and they can be even higher in real-world deployments\footnote{\url{https://www.macobserver.com/analysis/google-apple-differential-privacy/}}.

This guarantee does not tell us anything other than that this outlier cannot be protected while preserving utility. But what is the guarantee for other residents of the town? Intuitively, it should be much stronger. In the next section, we present a novel DP-based privacy notion. It uses the same privacy mechanism and augments the general DP guarantee with a much tighter guarantee for the expected case, and, by extension, for any percentile of the user/data population.

\section{Bayesian Differential Privacy}
\label{sec:bayes_dp}
In this section, we define \emph{Bayesian differential privacy (BDP)}. We then derive a practical privacy loss accounting method, and discuss its relation to the moments accountant. All the proofs are available in the supplementary material.

\subsection{Definition}
\label{sec:bdp_definition}
Let us define \emph{strong} Bayesian differential privacy (Definition~\ref{def:strong_bayes_dp}) and (\emph{weak}) Bayesian differential privacy (Definition~\ref{def:bayes_dp}). The first provides a better intuition, connection to concentration inequalities, and is being used for privacy accounting. Unfortunately, it may not be closed under post-processing, and therefore, the actual guarantee provided by BDP is stated in Definition~\ref{def:bayes_dp} and mimics the $(\varepsilon, \delta)$-differential privacy~\cite{dwork2014algorithmic}. The reason Definition~\ref{def:strong_bayes_dp} may pose a problem with post-processing is that it does not consider sets of outcomes, and a routine that integrates groups of values into one value could therefore invalidate the guarantee by increasing the probability ratio beyond epsilon. On the other hand, it can still be used for accounting with adaptive composition, because in this context, every next step is conditioned on a single outcome of the previous step. This separation mirrors the moments accountant approach of bounding tails of the privacy loss random variable and converting it to the $(\varepsilon, \delta)$-DP guarantee~\cite{abadi2016deep}, but does so in a more explicit manner.

\begin{definition}[Strong Bayesian Differential Privacy]
\label{def:strong_bayes_dp}
A randomised function $\mathcal{A}: \mathcal{D} \rightarrow \mathcal{R}$ with domain $\mathcal{D}$, range $\mathcal{R}$, and outcome $w=\mathcal{A}(\cdot)$, satisfies $(\varepsilon_\mu, \delta_\mu)$-strong Bayesian differential privacy if for any two adjacent datasets $D, D' \in \mathcal{D}$, \emph{differing in a single data point $x' \sim \mu(x)$}, the following holds:
\begin{align}
\label{eq:strong_bayes_dp}
	\Pr[L_\mathcal{A}(w, D, D') \geq \varepsilon_\mu] \leq \delta_\mu,
\end{align}
where probability is taken over the randomness of the outcome $w$ and the additional example $x'$.
\end{definition}

Here, $L_\mathcal{A}(w, D, D')$ is the privacy loss defined as
\begin{align}
	L_{\mathcal{A}}(w, D, D') = \log\frac{p(w|D)}{p(w|D')},
\end{align}
where $p(w|D)$, $p(w|D')$ are private outcome distributions for corresponding datasets. For brevity, we often omit parameters and denote the privacy loss simply by $L$.

We use the subscript $\mu$ to underline the main difference between the classic DP and Bayesian DP: in the classic definition the probability is taken only over the randomness of the outcome ($w$), while the BDP definition contains two random variables ($w$ and $x'$). Therefore, the privacy parameters $\varepsilon$ and $\delta$ depend on the data distribution $\mu(x)$.

The addition of another random variable yields the change in the meaning of $\delta_\mu$ compared to the $\delta$ of DP. In Bayesian differential privacy, it also accounts for the privacy mechanism failures in the tails of data distributions in addition to the tails of outcome distributions.

\begin{definition}[Bayesian Differential Privacy]
\label{def:bayes_dp}
A randomised function $\mathcal{A}: \mathcal{D} \rightarrow \mathcal{R}$ with domain $\mathcal{D}$ and range $\mathcal{R}$ satisfies $(\varepsilon_\mu, \delta_\mu)$-Bayesian differential privacy if for any two adjacent datasets $D, D' \in \mathcal{D}$, \emph{differing in a single data point $x' \sim \mu(x)$}, and for any set of outcomes $\mathcal{S}$ the following holds:
\begin{align}
\label{eq:bayes_dp}
	\Pr\left[\mathcal{A}(D) \in \mathcal{S} \right] \leq e^{\varepsilon_\mu} \Pr\left[\mathcal{A}(D') \in \mathcal{S} \right] + \delta_\mu.
\end{align}
\end{definition}

\begin{restatable}{proposition}{weakBDP}
\label{prop:weakBDP}
$(\varepsilon_\mu, \delta_\mu)$-strong Bayesian differential privacy implies $(\varepsilon_\mu, \delta_\mu)$-Bayesian differential privacy.
\end{restatable}

Bayesian DP repeats some basic properties of the classic DP, such as composition, post-processing resilience and group privacy. More details, proofs for these properties and the above proposition, can be found in supplementary material.

While Definitions~\ref{def:strong_bayes_dp} and~\ref{def:bayes_dp} do not specify the distribution of any point in the dataset other than the additional example $x'$, it is natural to assume that all examples in the dataset are drawn from the same distribution $\mu(x)$. This holds in many real-world applications, including applications evaluated in this paper, and it allows using dataset samples instead of requiring knowing the true distribution.

We also assume that data points are exchangeable~\cite{aldous1985exchangeability}, i.e. any permutation of data points has the same joint probability. It enables tighter accounting for iterative applications of the privacy mechanism~(see Section~\ref{sec:accounting}), is weaker than independence and is naturally satisfied in the considered scenarios.

\subsection{Privacy Accounting}
\label{sec:accounting}
In the context of learning, it is important to be able to keep track of the privacy loss over iterative applications of the privacy mechanism. And since the bounds provided by the basic composition theorem are loose, we formulate the \emph{advanced composition theorem} and develop a general accounting method for Bayesian differential privacy, the \emph{Bayesian accountant}, that provides a tight bound on privacy loss and is straightforward to implement. We draw inspiration from the moments accountant~\cite{abadi2016deep}.

Observe that Eq.~\ref{eq:bayes_dp} is a typical concentration bound inequality, which are well studied in probability theory. One of the most common examples of such bounds is Markov's inequality. In its extended form, it states the following:
\begin{align}
\label{eq:concentration}
	\Pr[|L| \geq \varepsilon_\mu] \leq \frac{\mathbb{E}[\varphi(|L|)]}{\varphi(\varepsilon_\mu)},
\end{align}
where $\varphi(\cdot)$ is a monotonically increasing non-negative function. It is immediately evident that it provides a relation between $\varepsilon_\mu$ and $\delta_\mu$ (i.e. $\delta_\mu = \mathbb{E}[\varphi(|L|)] / \varphi(\varepsilon_\mu)$), and in order to determine them we need to choose $\varphi$ and compute the expectation $\mathbb{E}[\varphi(|L(w, D, D')|)]$. Note that $L(w, D, D') = - L(w, D', D)$, and since the inequality has to hold for any pair of $D, D'$, we can use $L$ instead of $|L|$.

We use the Chernoff bound that can be obtained by choosing $\varphi(L) = e^{\lambda L}$. It is widely known because of its tightness, and although not explicitly stated, it is also used by \citet{abadi2016deep}. The inequality in this case transforms to
\begin{align}
\label{eq:chernoff_bound}
	\Pr[L \geq \varepsilon_\mu] \leq \frac{\mathbb{E}[e^{\lambda L}]}{e^{\lambda\varepsilon_\mu}}.
\end{align}

This inequality requires the knowledge of the moment generating function of $L$ or some bound on it. The choice of the parameter $\lambda$ can be arbitrary, because the bound holds for any value of it, but it determines how tight the bound is. By simple manipulations we obtain 
\begin{align}
\label{eq:mgf}
	\mathbb{E}[e^{\lambda L}] &= \mathbb{E}\left[e^{\lambda \log\frac{p(w|D)}{p(w|D')}}\right] \nonumber \\
		&= \mathbb{E}\left[\left(\frac{p(w|D)}{p(w|D')}\right)^\lambda \right].
\end{align}

If the expectation is taken only over the outcome randomness, this expression is the function of R\'enyi divergence between $p(w|D)$ and $p(w|D')$, and following this path yields re-derivation of R\'enyi differential privacy~\cite{mironov2017renyi}. However, by also taking the expectation over additional examples $x' \sim \mu(x)$, we can further tighten this bound.

By the law of total expectation,
\begin{align}
\label{eq:total_expectation}
	\mathbb{E}\left[\left(\frac{p(w|D)}{p(w|D')}\right)^\lambda \right] = 
		\mathbb{E}_x \left[ \mathbb{E}_w \left[ \left. \left( \frac{p(w|D)}{p(w|D')}\right)^\lambda \right\vert x' \right] \right],
\end{align}
where the inner expectation is again the function of R\'enyi divergence, and the outer expectation is over $\mu(x)$.

Combining Eq.~\ref{eq:mgf} and~\ref{eq:total_expectation} and plugging it in Eq.~\ref{eq:chernoff_bound}, we get
\begin{align}
\label{eq:better_bound}
	\Pr[L \geq \varepsilon_\mu] \leq \mathbb{E}_x \left[ e^{\lambda \mathcal{D}_{\lambda+1} [p(w|D) \| p(w|D')] - \lambda \varepsilon_\mu} \right].
\end{align}

This expression determines how to compute $\varepsilon_\mu$ for a fixed $\delta_\mu$ (or vice versa) for one invocation of the privacy mechanism. However, to accommodate the iterative nature of learning, we need to deal with the composition of multiple applications of the mechanism. We already determined that our privacy notion is naively composable, but in order to achieve better bounds we need a tighter composition theorem. 

\begin{restatable}[Advanced Composition]{theorem}{composition}
\label{thm:advanced_composition}
Let a learning algorithm run for $T$ iterations. Denote by $w^{(1)} \ldots w^{(T)}$ a sequence of private learning outcomes at iterations $1,\ldots,T$, and $L^{(1:T)}$ the corresponding total privacy loss. Then,
\begin{align*}
	\mathbb{E}\left[e^{\lambda L^{(1:T)}}\right] \leq \prod_{t=1}^T \mathbb{E}_x \left[ e^{T\lambda \mathcal{D}_{\lambda+1} (p_t \| q_t)} \right]^{\frac{1}{T}},
\end{align*}
where $p_t = p(w^{(t)} | w^{(t-1)}, D)$, $q_t = p(w^{(t)} | w^{(t-1)}, D')$.
\end{restatable}
\begin{proof}
See supplementary material.
\end{proof}

Unlike the moments accountant, our composition theorem presents an upper bound on the total privacy loss due to computing expectation over the distribution of the same example over all iterations. However, we found that the inequality tends to be tight in practice, and there is little overhead compared to na\"ively swapping the product and the expectation.

We denote the logarithm of the quantity inside the product in Theorem~\ref{thm:advanced_composition} as $c_t(\lambda, T)$ and call it the \emph{privacy cost} of the iteration $t$:
\begin{align}
\label{eq:privacy_cost}
	c_t(\lambda, T) = \log \mathbb{E}_x \left[ e^{T \lambda \mathcal{D}_{\lambda+1} (p_t \| q_t)} \right]^{\frac{1}{T}}
\end{align}

The privacy cost of the whole learning process is then a sum of the costs of each iteration. We can now relate $\varepsilon$ and $\delta$ parameters of BDP through the privacy cost.
\begin{theorem}
\label{thm:eps_delta_relation}
Let the algorithm produce a sequence of private learning outcomes $w^{(1)} \ldots w^{(T)}$ using a known probability distribution $p(w^{(t)} | w^{(t-1)}, D)$. Then, for a fixed $\varepsilon_\mu$:
\begin{align*}
	\log \delta_\mu \leq \sum_{t=1}^T c_t(\lambda, T) - \lambda \varepsilon_\mu.
\end{align*}
\end{theorem}

\begin{corollary}
\label{thm:eps_from_delta}
Under the conditions above, for a fixed $\delta_\mu$:
\begin{align*}
	\varepsilon_\mu \leq \frac{1}{\lambda} \sum_{t=1}^T c_t(\lambda, T) - \frac{1}{\lambda} \log \delta_\mu.
\end{align*}
\end{corollary}

Theorems~\ref{thm:advanced_composition},~\ref{thm:eps_delta_relation} and Corollary~\ref{thm:eps_from_delta} immediately provide us with an efficient privacy accounting algorithm. During training, we compute the privacy cost $c_t(\lambda, T)$ for each iteration $t$, accumulate it, and then use to compute $\varepsilon_\mu, \delta_\mu$ pair. This process is ideologically close to that of the moment accountant but accumulates a different quantity (note the change from the privacy loss random variable to R\'enyi divergence). We further explore this connection in Section~\ref{sec:relation_to_ma_rdp}.

The link to R\'enyi divergence is an advantage for applicability of this framework: if the outcome distribution $p(w|D)$ has a known analytic expression for R\'enyi divergence~\cite{gil2013renyi,van2014renyi}, it can be easily plugged into our accountant.

For the popular subsampled Gaussian mechanism~\cite{abadi2016deep}, we can demonstrate the following.
\begin{restatable}{theorem}{gaussian}
\label{thm:gauss_privacy_cost}
Given the Gaussian noise mechanism with the noise parameter $\sigma$ and subsampling probability $q$, the privacy cost for $\lambda \in \mathbb{N}$ at iteration $t$ can be expressed as
\begin{align*}
	c_t(\lambda, T) = \max\{c_t^{L}(\lambda, T), c_t^{R}(\lambda, T)\},
\end{align*}
where
\begin{align*}
	&c_t^{L}(\lambda, T) = \frac{1}{T} \log \mathbb{E}_x \left[ \mathbb{E}_{k \sim B(\lambda+1, q)} \left[e^{\frac{k^2 - k}{2\sigma^2} \|g_t - g'_t\|^2} \right]^T \right], \\
	&c_t^{R}(\lambda, T) = \frac{1}{T} \log \mathbb{E}_x \left[ \mathbb{E}_{k \sim B(\lambda, q)} \left[e^{\frac{k^2 + k}{2\sigma^2} \|g_t - g'_t\|^2} \right]^T \right],
\end{align*}
and $B(\lambda, q)$ is the binomial distribution with $\lambda$ experiments and the probability of success $q$. 
\end{restatable}

\subsection{Privacy Cost Estimator}
\label{sec:privacy_cost_estimator}
Computing $c_t(\lambda, T)$ precisely requires access to the data distribution $\mu(x)$, which is unrealistic. Therefore, we need an estimator for $\mathbb{E}[e^{T\lambda \mathcal{D}_{\lambda+1} (p_t \| q_t)}]$.

Typically, having access to the distribution samples, one would use the law of large numbers and approximate the expectation with the sample mean. This estimator is unbiased and converges with the growing number of samples. However, these are not the properties we are looking for. The most important property of the estimator in our context is that it \emph{does not underestimate} $\mathbb{E}[e^{T\lambda \mathcal{D}_{\lambda+1} (p_t \| q_t)}]$, because the bound (Eq.~\ref{eq:chernoff_bound}) would not hold for this estimate otherwise. 

We employ the Bayesian view of the parameter estimation problem~\cite{oliphant2006bayesian} and design an estimator with this single property: given a fixed $\gamma$, it returns the value that overestimates the true expectation with probability $1 - \gamma$. We then incorporate the estimator uncertainty $\gamma$ in $\delta_\mu$.

\subsubsection{Binary Case}
Let us demonstrate the process of constructing the expectation estimator with the above property on a simple binary example. This technique is based on~\cite{oliphant2006bayesian} and it translates directly to other classes of distributions with minor adjustments. Here, we also address the concern of not taking into account the data absent from the dataset.

Let the data $\{x_1, x_2, \ldots, x_N\}$, such that $x_i \in \{0, 1\}$, have a common mean and a common variance. As this information is insufficient to solve our problem, let us also assume that the data comes from \emph{the maximum entropy distribution}. This assumption adds the minimum amount of information to the problem and makes our estimate pessimistic.

For the binary data with the common mean $\rho$, the maximum entropy distribution is the Bernoulli distribution:
\begin{align}
	f(x_i | \rho) = \rho^{x_i} (1 - \rho)^{1 - x_i},
\end{align}
where $\rho$ is also the probability of success ($x_i = 1$). Then,
\begin{align}
	f(x_1, \ldots, x_N | \rho) = \rho^{N_1} (1 - \rho)^{N_0},
\end{align}
where $N_0$ and $N_1$ is the number of $0$'s and $1$'s in the dataset.

We impose the flat prior on $\rho$, assuming all values in $[0, 1]$ are equally likely, and use Bayes' theorem to determine the distribution of $\rho$ given the data:
\begin{align}
	f(\rho | x_1, \ldots, x_N) = \frac{\Gamma(N_0 + N_1 +2)}{\Gamma(N_0 + 1)\Gamma(N_1 + 1)} \rho^{N_1} (1 - \rho)^{N_0},
\end{align}
where the normalisation constant in front is obtained by setting the integral over $\rho$ equal to $1$.

Now, we can use the above distribution of $\rho$ to design an estimator $\hat{\rho}$, such that it overestimates $\rho$ with high probability, i.e. $\Pr\left[\rho \leq \hat{\rho}\right] \geq 1 - \gamma$. Namely, $\hat{\rho} = F^{-1}(1-\gamma)$, where $F^{-1}$ is the inverse of the CDF: 
\begin{align*}
	F^{-1}&(1-\gamma) \\ 
		&= \inf\{z \in \mathbb{R} : \int_{-\infty}^z f(t | x_1, \ldots, x_N) dt \geq 1 - \gamma \}.
\end{align*}
We refer to $\gamma$ as the \emph{estimator failure probability}, and to $1 - \gamma$ as the \emph{estimator confidence}.

To demonstrate the resilience of this estimator to unseen data, consider the following example. Let the true expectation be $0.01$, and let the data consist of 100 zeros, and no ones. A typical "frequentist" mean estimator would confidently output $0$. However, our estimator would never output $0$, unless the confidence is set to $0$. When the confidence is set to $1$ ($\gamma = 0$), the output is $1$, which is the most pessimistic estimate. Finally, the output $\hat{\rho} = \rho = 0.01$ will be assigned the failure probability $\gamma = 0.99^{101} \approx 0.36$, which is the probability of not drawing a single $1$ in $101$ draws. 

In a real-world system, the confidence would be set to a much higher level (in our experiments, we use $\gamma = 10^{-15}$), and the probability of $1$ would be significantly overestimated. Thus, unseen data do not present a problem for this estimator, because it exaggerates the probability of data that increase the estimated expectation.

\subsubsection{Continuous Case}
For applications evaluated in this paper, we are primarily concerned with continuous case. Thus, let us define the following $m$-sample estimator of $c_t(\lambda, T)$ for continuous distributions with existing mean and variance:
\begin{align}
	\hat{c}_t(\lambda, T) = \log \left[ M(t) + \frac{F^{-1}(1 - \gamma, m - 1)}{\sqrt{m - 1}} S(t) \right],
\end{align}
where $M(t)$ and $S(t)$ are the sample mean and the sample standard deviation of $e^{\lambda \hat{D}_{\lambda+1}^{(t)} }$, $F^{-1}(1-\gamma, m - 1)$ is the inverse of the Student's $t$-distribution CDF at $1-\gamma$ with $m - 1$ degrees of freedom, and
\begin{align}
	& \hat{D}_{\lambda+1}^{(t)}  = \max \left\{ D_{\lambda+1} (\hat{p}_t \| \hat{q}_t),~D_{\lambda+1} (\hat{q}_t \| \hat{p}_t) \right\}, \nonumber \\
	& \hat{p}_t = p(w^{(t)}~|~w^{(t-1)}, B^{(t)}), \nonumber \\
	& \hat{q}_t = p(w^{(t)}~|~w^{(t-1)}, B^{(t)} \setminus \{x_i\}).
\end{align}
Since in many cases learning is performed on mini-batches, we can similarly compute R\'enyi divergence on batches $B^{(t)}$.

\begin{restatable}{theorem}{estimator}
\label{thm:estimator}
Estimator $\hat{c}_t(\lambda, T)$ overestimates $c_t(\lambda, T)$ with probability $1-\gamma$. That is,
\begin{align*}
	\Pr \left[ \hat{c}_t(\lambda, T) < c_t(\lambda, T) \right] \leq \gamma.
\end{align*}
\end{restatable}
The proof follows the steps of the binary example above.

\begin{remark}
By adapting the maximum entropy probability distribution an equivalent estimator can be derived for other classes of distributions (e.g. discrete).
\end{remark}

To avoid introducing new parameters in the privacy definition, we can incorporate the probability $\gamma$ of underestimating the true expectation  in $\delta_\mu$. We can re-write:

\begin{align}
	\Pr&[L_\mathcal{A}(w^{(t)}, D, D') \geq \varepsilon_\mu] \nonumber \\
		&= \Pr\left[ L_\mathcal{A}(w^{(t)}, D, D') \geq \varepsilon_\mu, \hat{c}_t(\lambda, T) \geq c_t(\lambda, T) \right] \nonumber \\
		&\quad + \Pr\left[ L_\mathcal{A}(w^{(t)}, D, D') \geq \varepsilon_\mu, \hat{c}_t(\lambda, T) < c_t(\lambda, T) \right].
\end{align}

When $\hat{c}_t(\lambda, T) \geq c_t(\lambda, T)$, using the Chernoff inequality, the first summand is bounded by $\beta = \exp(\sum_{t=1}^T \hat{c}_t(\lambda, T) - \lambda \varepsilon_\mu)$.

Whenever $\hat{c}_t(\lambda, T) < c_t(\lambda, T)$, 
\begin{align}
	\Pr&[ L_\mathcal{A}(w^{(t)}, D, D') \geq \varepsilon_\mu, \hat{c}_t(\lambda, T) < c_t(\lambda, T) ] \nonumber \\
		&\leq \Pr[ \hat{c}_t(\lambda, T) < c_t(\lambda, T) ] \nonumber \\
		&\leq \gamma.
\end{align}

Therefore, the true $\delta_\mu$ is bounded by $\beta + \gamma$, and despite the incomplete data, we can claim that the mechanism is $(\varepsilon_\mu, \delta_\mu)$-Bayesian differentially private, where $\delta_\mu = \beta + \gamma$.

\begin{remark}
This step further changes the interpretation of $\delta_\mu$ in Bayesian differential privacy compared to the classic $\delta$ of DP. Apart from the probability of the privacy loss exceeding $\varepsilon_\mu$, e.g. in the tails of its distribution, it also incorporates our uncertainty about the true data distribution (in other words, the probability of underestimating the true expectation because of not observing enough data samples). It can be intuitively understood as accounting for unobserved (but feasible) data in $\delta_\mu$, rather than in $\varepsilon_\mu$.
\end{remark}

\subsection{Discussion}

\subsubsection{Relation to DP}
\label{sec:relation_to_dp}
To better understand how the BDP bound relates to the traditional DP, consider the following conditional probability:
\begin{align}
\Delta(\varepsilon, x') = \Pr\left[ L(w, D, D') > \varepsilon ~|~ D, D' = D \cup \{x'\} \right].
\end{align}
The moments accountant outputs $\delta$ that upper-bounds $\Delta(\varepsilon, x')$ for all $x'$. It is not true in general for other accounting methods, but let us focus on MA, as it is by far the most popular. Consequently, the MA bound is
\begin{align}
\max_{x} \Delta(\varepsilon, x) \leq \delta,
\end{align}
where $\varepsilon$ is a chosen constant. At the same time, BDP bounds the probability that is not conditioned on $x'$, but we can transform one to another through marginalisation and get:
\begin{align}
\mathbb{E}_{x} \left[ \Delta(\varepsilon, x) \right] \leq \delta_\mu.
\end{align}
Since $\Delta(\cdot)$ is a non-negative random variable in $x$, we can apply Markov's inequality and obtain a tail bound on it using $\delta_\mu$. \emph{We can therefore find a pair $(\varepsilon, \delta)_p$ that holds for any percentile $p$ of the data distribution, not just in expectation.} In all our experiments, we consider bounds well above 99th percentile, so it is very unlikely to encounter data for which the equivalent DP guarantee doesn't hold. Moreover, it is possible to characterise privacy by building a curve for different percentiles, and hence, gain more insight into how well users and their data are protected.

\subsubsection{Relation to Moments Accountant}
\label{sec:relation_to_ma_rdp}
As mentioned in Section~\ref{sec:accounting}, removing the distribution requirement on $D, D'$ and further simplifying Eq.~\ref{eq:better_bound}, we can recover the relation between R\'enyi DP and $(\varepsilon, \delta)$-DP.

At the same time, our accounting technique closely resembles the moments accountant. In fact, we can show that the moments accountant is a special case of Theorem~\ref{thm:gauss_privacy_cost}. Ignoring the data distribution information and substituting expectation by $\max_{D,D'}$ yields the substitution of $\|g_t - g'_t\|$ for $C$ in Theorem~\ref{thm:gauss_privacy_cost}, where $C$ is the sensitivity (or clipping threshold), which turns out to be the exact moments accountant bound. In addition, there are some extra benefits, such as avoiding numerical integration when using $\lambda \in \mathbb{N}$ due to connection to Binomial distribution, which improves numerical stability and computational efficiency.

\subsubsection{Privacy of $\hat{c}_t(\lambda, T)$}
\label{sec:privacy_of_estimator}

\begin{figure*}
	\centering
	\begin{subfigure}{0.245\textwidth}
    		\includegraphics[trim={10pt 0 8pt 0},clip,width=\textwidth]{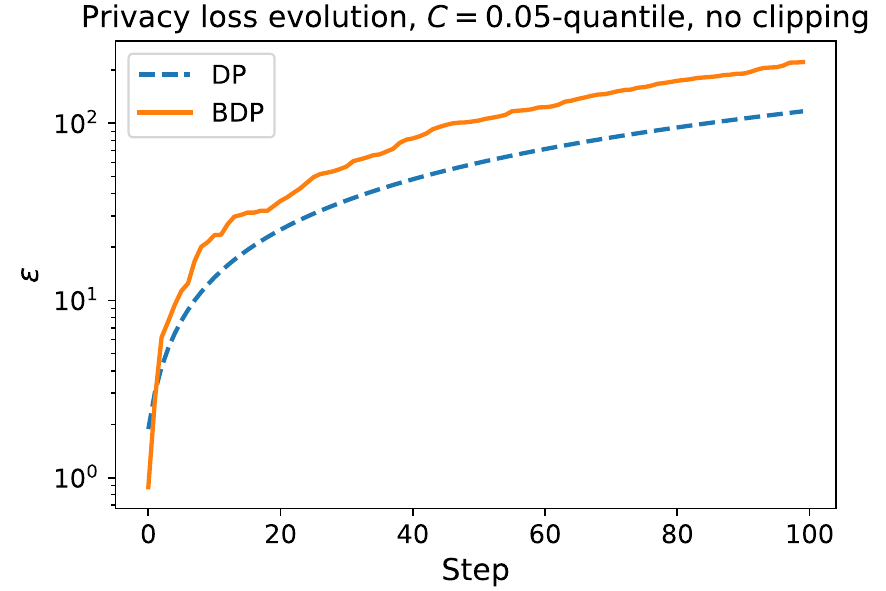}
    		\caption{$0.05$-quantile of $\|\nabla f\|$.}
	\end{subfigure}
	\begin{subfigure}{0.245\textwidth}
    		\includegraphics[trim={10pt 0 8pt 0},clip,width=\textwidth]{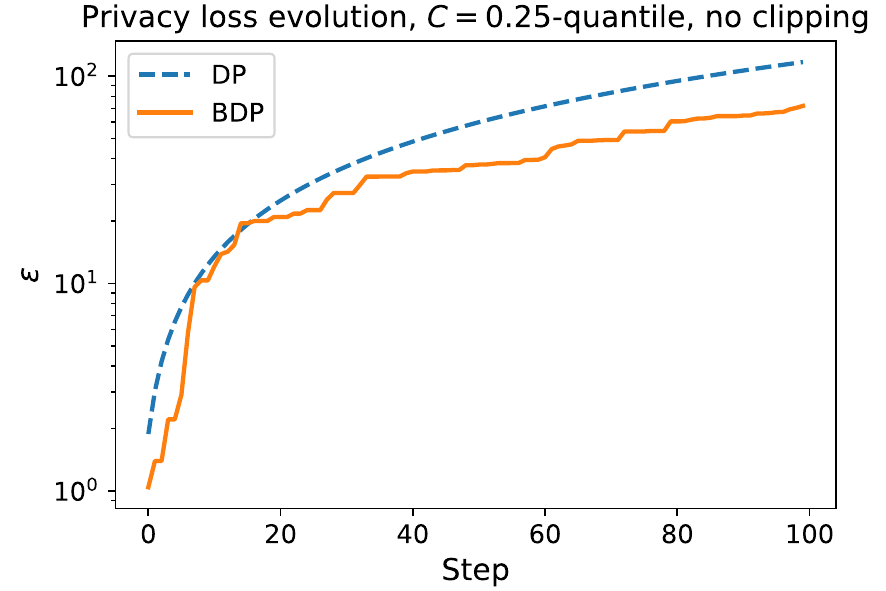}
    		\caption{$0.25$-quantile of $\|\nabla f\|$.}
	\end{subfigure}
	\begin{subfigure}{0.245\textwidth}
    		\includegraphics[trim={10pt 0 8pt 0},clip,width=\textwidth]{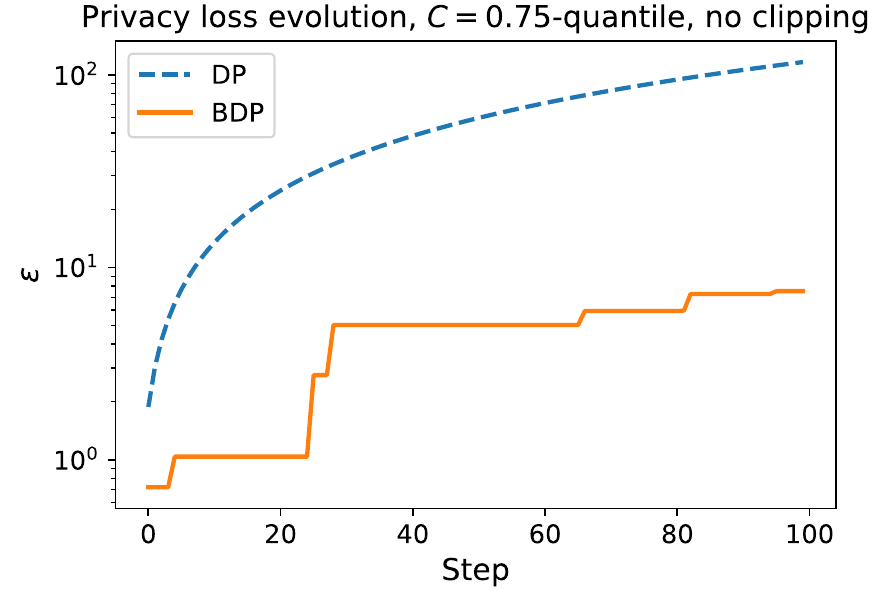}
    		\caption{$0.75$-quantile of $\|\nabla f\|$.}
	\end{subfigure}
	\begin{subfigure}{0.245\textwidth}
    		\includegraphics[trim={10pt 0 8pt 0},clip,width=\textwidth]{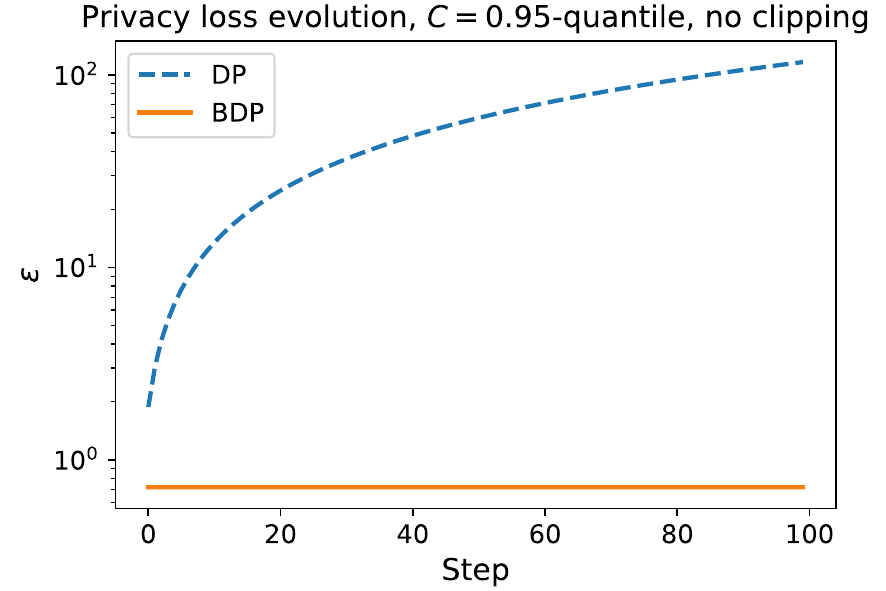}
    		\caption{$0.95$-quantile of $\|\nabla f\|$.}
    		\label{fig:eps_step_no_clipping_d}
	\end{subfigure}
	\caption{Evolution of $\varepsilon$ and $\varepsilon_\mu$ over multiple steps of the Gaussian noise mechanism with $\sigma = C$ for DP (with clipping) and BDP (without clipping). Sub-captions indicate the noise variance relative to the gradient norms distribution.}
	\label{fig:eps_step_no_clipping}
\end{figure*}

Due to calculating $\hat{c}_t(\lambda, T)$ from data, our privacy guarantee becomes data-dependent and may potentially leak information. To obtain a theoretical bound on this leakage, we need to get back to the maximum entropy assumption in Section~\ref{sec:privacy_cost_estimator}, and assume that $M(t)$ and $S(t)$ are following some specific distributions, such as Gaussian and $\chi^2$ distributions. Hence, in case of simple random sampling, these statistics for two neighbour datasets are differentially private and the privacy parameters can be computed using R\'enyi divergence. Furthermore, these guarantees are controlled by the number of samples used to compute the statistics: the more samples are used, the more accurate the statistics are, and the less privacy leakage occurs. This property can be used to control estimates privacy without sacrificing their tightness, only at the cost of extra computation time. Without distributional assumptions, the bound can be computed in the limit of the sample size used by the estimator, using the CLT.

On the other hand, consider the fact that the information from many high-dimensional vectors gets first compressed down to their pairwise distances, which are not as informative in high-dimensional spaces (i.e. the curse of dimensionality), and then down to one number. Intuitively, at this rate of compression very little knowledge can be gained by an attacker in practice.

The first approach would provide little information about real-world cases due to potentially unrealistic assumptions, and hence, we opt for the second approach. We examine pairwise gradient distances of the points within the training set and outside, and demonstrate that the privacy leakage is not statistically significant in practice (see Section~\ref{sec:deep_learning}).

\section{Evaluation}
\label{sec:evaluation}

This experimental section comprises two parts. First, we examine how well Bayesian DP composes over multiple steps. We use the Bayesian accountant and compare to the state-of-the-art DP results obtained by the moments accountant. Second, we consider the context of machine learning. In particular, we use the differentially private stochastic gradient descent (DP-SGD), a well known privacy-preserving learning technique broadly used in combination with the moments accountant, to train neural networks on classic image classification tasks MNIST~\cite{lecun1998gradient} and CIFAR10~\cite{krizhevsky2009learning}. We then compare the accuracy and privacy guarantees obtained under BDP and under DP. We also perform experiments with variational inference on Abalone~\cite{waugh1995extending} and Adult~\cite{kohavi1996scaling} datasets.

Importantly, DP and BDP can use the same privacy mechanism and be accounted in parallel to ensure the DP guarantees hold if BDP assumptions fail. Thus, all comparisons in this section should be viewed in the following way: the reported BDP guarantee would apply to \emph{typical} data (i.e. data drawn from the same distribution as the dataset); the reported DP guarantee would apply to all other data; their difference is the advantage for typical data we gain by using Bayesian DP. In some experiments we use smaller noise variance for BDP in order to speed up training, meaning that the reported BDP guarantees will further improve if noise variance is increased to DP levels. For more details and additional experiments, we refer the reader to the supplementary material, while the source code is available on GitHub\footnote{\url{https://github.com/AlekseiTriastcyn/bayesian-differential-privacy}}.

\begin{table}
	\setlength{\tabcolsep}{3pt}
	\caption{Estimated privacy bounds $\varepsilon$ for $\delta=10^{-5}$ and $\delta_\mu = 10^{-10}$ for MNIST, CIFAR10, Abalone and Adult datasets. In parenthesis, a potential attack success probability $P(A)$.}
	\label{tab:privacy}
	\centering
	\begin{tabular}{l | c c | c c}
		\toprule
		 						& \multicolumn{2}{c|}{\bf Accuracy}	& \multicolumn{2}{c}{\bf Privacy} 	\\
		{\bf Dataset} 	& {Baseline} 			& {Private}			& {DP} 						& {BDP} 					\\
		\midrule
		MNIST 				& $99\%$ 			& $96\%$				& $2.2~(0.898)$ 		& \boldmath${0.95~(0.721)}$	\\
		CIFAR10 			& $86\%$ 			& $73\%$				& $8.0~(0.999)$		& \boldmath${0.76~(0.681)}$		\\
		Abalone 			& $77\%$ 				& $76\%$				& $7.6~(0.999)$		& \boldmath${0.61~(0.649)}$		\\
		Adult 				& $81\%$ 				& $81\%$				& $0.5~(0.623)$ 		& \boldmath${0.2~(0.55)}$		\\
		\bottomrule
	\end{tabular}
\end{table}

\subsection{Composition}
\label{sec:composition}
First, we study the growth rate of the privacy loss over a number of mechanism invocations. This experiment is carried out using synthetic gradients drawn from the Weibull distribution with the shape parameter $< 1$ to imitate a more difficult case of heavy-tailed gradient distributions. We do not clip gradients for BDP in order to show the raw effect of the signal-to-noise ratio on the privacy loss behaviour. Technically, bounded sensitivity is not as essential for BDP, because extreme individual contributions are mitigated by their low probability. However, in practice it is still advantageous to have a better control over privacy loss spikes and ensure that the worst-case DP guarantee is preserved.

In Figure~\ref{fig:eps_step_no_clipping}, we plot $\varepsilon$ and $\varepsilon_\mu$ as a function of steps for different levels of noise. Naturally, as the noise standard deviation gets closer to the expected gradients norm, the growth rate of the privacy loss decreases dramatically. Even when the noise is at the $0.25$-quantile, the Bayesian accountant matches the moments accountant. It is worth noting, that DP behaves the same in all these experiments because the gradients get clipped at the noise level $C$. Introducing clipping for BDP yields the behaviour of Figure~\ref{fig:eps_step_no_clipping_d}, as we demonstrate in the next section on real data.

\subsection{Learning}
\label{sec:deep_learning}
We now consider the application to privacy-preserving deep learning. Our setting closely mimics that of~\cite{abadi2016deep} to enable a direct comparison with the moments accountant and DP. We use a version of DP-SGD that has been extensively applied to build differentially private machine learning models. The idea of DP-SGD is to clip the gradient norm to some constant $C$ (ensuring bounded sensitivity) and add Gaussian noise with variance $C^2\sigma^2$ at every iteration of SGD. For Abalone and Adult, we use variational inference in a setting similar to~\cite{jalko2016differentially}.

Using the gradient distribution information allows the BDP models to reach the same accuracy at a much lower $\varepsilon$ (for $99.999\%$ of data points from this distribution, see Section~\ref{sec:relation_to_dp}). On MNIST, we manage to reduce it from $2.2$ to $0.95$. For CIFAR10, from $8.0$ to $0.76$. See details in Table~\ref{tab:privacy}. Moreover, since less noise is required for Bayesian DP, the models reach the same test accuracy much faster. For example, our model reaches $96\%$ accuracy within 50 epochs for MNIST, while DP model requires more noise and slower training over hundreds of epochs to avoid $\varepsilon$ blowing up. These results confirm that discounting outliers in the privacy accounting process is highly beneficial for getting high accuracy and tighter guarantees for all the other points. To make our results more transparent, we include in Table~\ref{tab:privacy} the potential attack success probability $P(A)$ computed using Eq.~\ref{eq:motivation}. In this interpretation, the benefits of using BDP become even more apparent.

An important aspect of BDP, discussed in Section~\ref{sec:privacy_of_estimator}, is the potential privacy leakage of the privacy cost estimator. To illustrate that this leakage is minimal, we conduct the following experiment. After training the model (to ensure it contains as much information about data as possible), we compute the gradient pairwise distances over train and test sets. We then plot the histograms of these distances to inspect any divergence that would distinguish the data that was used in training. Note that this is more information than what is available to an adversary, who only observes $\varepsilon_\mu$, $\delta_\mu$. 

\begin{figure}
	\centering
	\begin{subfigure}{0.495\linewidth}
		\includegraphics[trim={10pt 0 10pt 0},width=\textwidth]{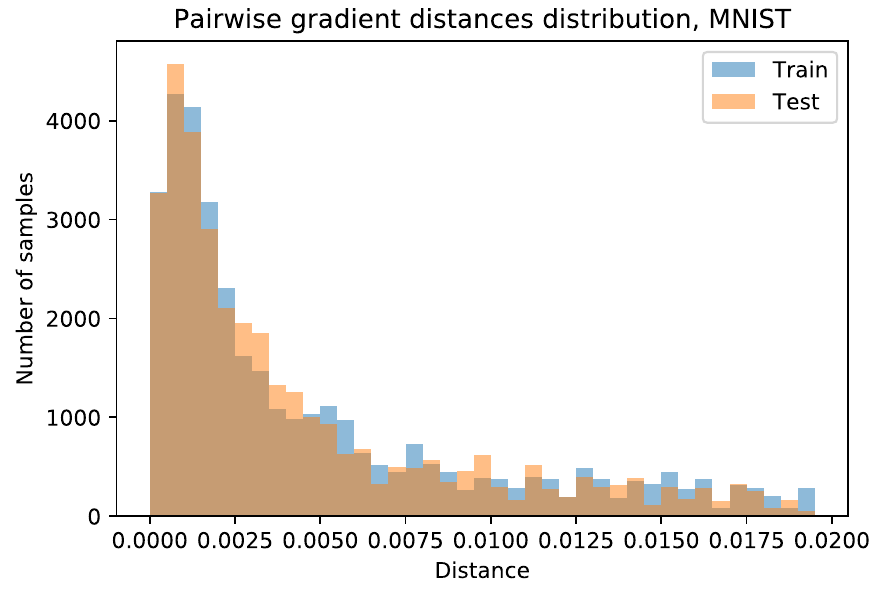}
		\caption{MNIST.}
		\label{fig:grad_hist_mnist}
	\end{subfigure}
	\begin{subfigure}{0.495\linewidth}
		\includegraphics[trim={10pt 0 10pt 0},width=\textwidth]{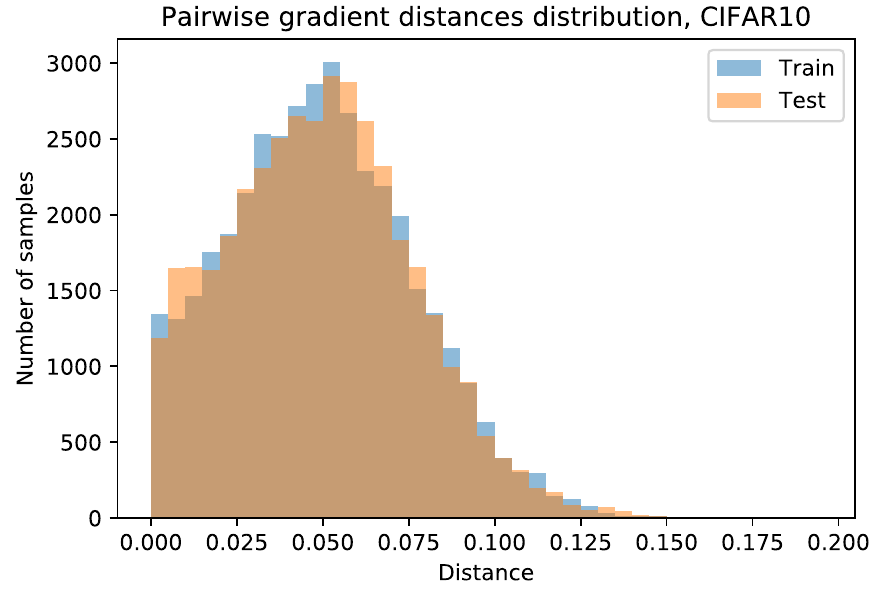}
		\caption{CIFAR10.}
		\label{fig:grad_hist_cifar10}
	\end{subfigure}
	\caption{Pairwise gradient distances distribution.}
\end{figure}

As it turns out, these distributions are nearly identical (see Figures~\ref{fig:grad_hist_mnist} and~\ref{fig:grad_hist_cifar10}), and we do not observe any correlation with the fact of the presence of data in the training set. For example, the sample mean of the test set can be both somewhat higher or lower than that of the train set. We also run the $t$-test for equality of means and Levene's test for equality of variances, obtaining $p$-values well over the $0.05$ threshold, suggesting that the difference of the means and the variances of these distributions is not statistically significant and the equality hypothesis cannot be rejected.

\section{Conclusion}
\label{sec:conclusion}
We introduce the notion of $(\varepsilon_\mu, \delta_\mu)$-Bayesian differential privacy, a variation of $(\varepsilon, \delta)$-differential privacy for sensitive data that are drawn from an arbitrary (and unknown) distribution $\mu(x)$. It is a reasonable assumption in many ML scenarios where models are designed for and trained on specific data distributions (e.g. emails, face images, ECGs, etc.). For example, trying to hide music records in a training set for ECG analysis may be unjustified, because the probability of it appearing is actually much smaller than $\delta$.

We present the advanced composition theorem for Bayesian DP that allows for efficient and tight privacy accounting. Since the data distribution is unknown, we design an estimator that overestimates the privacy loss with high, controllable probability. Moreover, as the data sample is finite, we employ the Bayesian parameter estimation approach with the flat prior and the maximum entropy principle to avoid underestimating probabilities of unseen examples. 

Our evaluation confirms that Bayesian DP is highly beneficial in ML context where its additional assumptions are naturally satisfied. First, it needs less noise for comparable privacy guarantees (with high probability, as per Section~\ref{sec:relation_to_dp}). Second, models train faster and can reach higher accuracy. Third, it may be used along with DP to ensure the worst-case guarantee for out-of-distribution samples and outliers, while providing tighter guarantees for most cases. In our supervised learning experiments, $\varepsilon$ always remains below $1$, translating to much more meaningful bounds on a potential attacker success probability.

\bibliography{icml2020}
\bibliographystyle{icml2020}

\newpage

\appendix

\section*{Appendix}


\section{Proofs}

\subsection{Proofs of Propositions}
\label{app:propositions}

\weakBDP*
\begin{proof}
Let us define a set of outcomes for which the privacy loss variable exceeds the $\varepsilon$ threshold: $F(x') = \{ w : L_\mathcal{A}(w, D, D') > \varepsilon \}$, and its compliment $F^c(x')$.

We have,
\begin{align}
	\Pr[\mathcal{A}(D) \in \mathcal{S}] &= \int \Pr[\mathcal{A}(D) \in \mathcal{S}, x']~dx' \\
		&= \int \Pr[\mathcal{A}(D) \in \mathcal{S} \cap \mathcal{F}^c(x'), x'] \\
		&\qquad + \Pr[\mathcal{A}(D) \in \mathcal{S} \cap \mathcal{F}(x'), x']~dx' \\
		&= \int \Pr[\mathcal{A}(D) \in \mathcal{S} \cap \mathcal{F}^c(x') | x'] \mu(x') \\
		&\qquad	+ \Pr[\mathcal{A}(D) \in \mathcal{S} \cap \mathcal{F}(x'), x']~dx' \\
		&\leq \int e^\varepsilon \Pr[\mathcal{A}(D') \in \mathcal{S} \cap \mathcal{F}^c(x') | x'] \mu(x') \\
		&\qquad	+ \Pr[\mathcal{A}(D) \in \mathcal{S} \cap \mathcal{F}(x'), x']~dx' \\
		&\leq \int e^\varepsilon\Pr[\mathcal{A}(D') \in \mathcal{S}, x'] \\
		&\qquad	+ \Pr[\mathcal{A}(D) \in \mathcal{S} \cap \mathcal{F}(x'), x']~dx' \\
		&\leq e^\varepsilon\Pr[\mathcal{A}(D') \in \mathcal{S}] + \delta_\mu,
\end{align}
where we used the observation that $L \leq \varepsilon$ implies $\Pr[\mathcal{A}(D) \in \mathcal{S} \cap \mathcal{F}^c(x')] \leq e^\varepsilon \Pr[\mathcal{A}(D') \in \mathcal{S} \cap \mathcal{F}^c(x')]$, and therefore, $\Pr[\mathcal{A}(D) \in \mathcal{S} \cap \mathcal{F}^c(x') ~|~ x'] \leq e^\varepsilon \Pr[\mathcal{A}(D') \in \mathcal{S} \cap \mathcal{F}^c(x') ~|~ x']$, because $\mathcal{A}(D)$ does not depend on $x'$, and $\mathcal{A}(D')$ is already conditioned on $x'$ through $D'$. Additionally, in the first line we used marginalisation, and the last inequality is due to the fact that
\begin{align}
	\int &\Pr[\mathcal{A}(D) \in \mathcal{S} \cap \mathcal{F}(x'), x']~dx' \\
		&\leq \int \Pr[\mathcal{A}(D) \in \mathcal{F}(x'), x']~dx' \\
		&= \int \mu(x') \Pr[\mathcal{A}(D) \in \mathcal{F}(x') ~|~ x']~dx' \\
		&= \int \mu(x') \int_{w \in \mathcal{F}(x')} p_\mathcal{A}(w|D, x') ~dw~dx' \\
		&= \mathbb{E}_{x'} \left[ \mathbb{E}_w \left[ \mathds{1}\{L > \varepsilon\} \right] \right] \\
		&\leq \delta_\mu
\end{align}
\end{proof}

\begin{proposition}[Post-processing]
\label{prop:postprocessing}
Let $\mathcal{A}: \mathcal{D} \rightarrow \mathcal{R}$ be a $(\varepsilon_\mu, \delta_\mu)$-Bayesian differentially private algorithm. Then for any arbitrary randomised data-independent mapping $f : \mathcal{R} \rightarrow \mathcal{R}'$, $f(\mathcal{A}(D))$ is $(\varepsilon_\mu, \delta_\mu)$-Bayesian differentially private.
\end{proposition}
\begin{proof}
First, by Proposition~\ref{prop:weakBDP}, $(\varepsilon_\mu, \delta_\mu)$-strong BDP implies the weak sense of BDP:
\begin{align}
	\Pr\left[\mathcal{A}(D) \in \mathcal{S} \right] \leq e^{\varepsilon_\mu} \Pr\left[\mathcal{A}(D') \in \mathcal{S} \right] + \delta_\mu,
\end{align}
for any set of outcomes $\mathcal{S} \subset \mathcal{R}$.

For a data-independent function $f(\cdot)$:
\begin{align}
	\Pr\left[f(\mathcal{A}(D)) \in \mathcal{T} \right] &= \Pr\left[\mathcal{A}(D) \in \mathcal{S} \right] \\
		&\leq e^{\varepsilon_\mu} \Pr\left[\mathcal{A}(D') \in \mathcal{S} \right] + \delta_\mu, \\
		&= e^{\varepsilon_\mu} \Pr\left[f(\mathcal{A}(D')) \in \mathcal{T} \right] + \delta_\mu
\end{align}
where $\mathcal{S} = f^{-1}[\mathcal{T}]$, i.e. $\mathcal{S}$ is the preimage of $\mathcal{T}$ under $f$.
\end{proof}

\begin{proposition}[Basic composition]
\label{prop:basiccomposition}
Let $\mathcal{A}_i : \mathcal{D} \rightarrow \mathcal{R}_i$, $\forall i=1..k$, be a sequence of  $(\varepsilon_\mu, \delta_\mu)$-Bayesian differentially private algorithms. Then their combination, defined as $\mathcal{A}_{1:k} : \mathcal{D} \rightarrow \mathcal{R}_1 \times \ldots \times \mathcal{R}_k$, is $(k\varepsilon_\mu, k\delta_\mu)$-Bayesian differentially private.
\end{proposition}
\begin{proof}
Let us denote $L = \log\frac{p(w_1, \ldots, w_k | D)}{p(w_1, \ldots, w_k | D')}$.

Also, let $L_i = \log\frac{p(w_i | D, w_{i-1}, \ldots, w_1)}{p(w_i | D', w_{i-1}, \ldots, w_1)}$. Then,
\begin{align}
	\Pr\left[ L \geq k\varepsilon_\mu \right] &= \Pr\left[ \sum_{i=1}^k L_i \geq k\varepsilon_\mu \right] \\
		&\leq \sum_{i=1}^k \Pr[L_i \geq \varepsilon_\mu] \\
		&\leq \sum_{i=1}^k \delta_\mu \\
		&\leq k\delta_\mu
\end{align}

For the weak sense of BDP, the proof follows the steps of \citet[Appendix~B]{dwork2014algorithmic}.

\end{proof}

\begin{proposition}[Group privacy]
\label{prop:group}
Let $\mathcal{A}: \mathcal{D} \rightarrow \mathcal{R}$ be a $(\varepsilon_\mu, \delta_\mu)$-Bayesian differentially private algorithm. Then for all pairs of datasets $D, D' \in \mathcal{D}$, differing in $k$ data points $x_1, \ldots, x_k$ s.t. $x_i \sim \mu(x)$ for $i=1..k$, $\mathcal{A}(D)$ is $(k\varepsilon_\mu, ke^{k\varepsilon_\mu}\delta_\mu)$-Bayesian differentially private.
\end{proposition}
\begin{proof}
Let us define a sequence of datasets $D^{i}$, $i=1..k$, s.t. $D = D^{0}$, $D' = D^{k}$, and $D^{i}$ and $D^{i-1}$ differ in a single example. Then,
\begin{align}
	&\frac{p(w | D)}{p(w | D')} = \frac{p(w | D^{0}) p(w | D^{1}) \ldots p(w | D^{k-1})}{p(w | D^{1}) p(w | D^{2}) \ldots p(w | D^{k})}
\end{align}
Denote $L_i = \log\frac{p(w | D^{i-1})}{p(w | D^{i})}$ for $i = 1..k$.

Finally, applying the definition of $(\varepsilon_\mu, \delta_\mu)$-Bayesian differential privacy,
\begin{align}
	 \Pr\left[ L \geq k\varepsilon_\mu \right] &= \Pr\left[ \sum_{i=1}^k L_i \geq k\varepsilon_\mu \right] \\
	 	&\leq \sum_{i=1}^k \Pr[L_i \geq \varepsilon_\mu] \\
		&\leq k\delta_\mu
\end{align}

For the weak sense of BDP,
\begin{align}
	 Pr\left[\mathcal{A}(D) \in \mathcal{S} \right] &\leq e^{\varepsilon_\mu} Pr\left[\mathcal{A}(D^1) \in \mathcal{S} \right] + \delta_\mu \\
	 	&\leq e^{\varepsilon_\mu} \left( e^{\varepsilon_\mu} Pr\left[\mathcal{A}(D^2) \in \mathcal{S} \right] + \delta_\mu \right) + \delta_\mu \\
	 	&\leq e^{2\varepsilon_\mu} Pr\left[\mathcal{A}(D^2) \in \mathcal{S} \right] + e^{\varepsilon_\mu}\delta_\mu + \delta_\mu \\
	 	&\leq \ldots \\
	 	&\leq e^{k\varepsilon_\mu} \Pr\left[\mathcal{A}(D^k) \in \mathcal{S} \right] + \frac{e^{k\varepsilon_\mu} - 1}{e^{\varepsilon_\mu} - 1} \delta_\mu \label{eq:delta_geometric_progression} \\
	 	&\leq e^{k\varepsilon_\mu} \Pr\left[\mathcal{A}(D^k) \in \mathcal{S} \right] + \frac{k\varepsilon_\mu e^{k\varepsilon_\mu}}{\varepsilon_\mu} \delta_\mu \label{eq:bound_exp_minus_1} \\
	 	&\leq e^{k\varepsilon_\mu} \Pr\left[\mathcal{A}(D') \in \mathcal{S} \right] + k e^{k\varepsilon_\mu} \delta_\mu,
\end{align}
where in \eqref{eq:delta_geometric_progression} we use the formula for the sum of a geometric progression; in \eqref{eq:bound_exp_minus_1}, the facts that $e^x - 1 \leq xe^x$, for $x > 0$, and $e^x \geq x + 1$.
\end{proof}

\subsection{Proof of Theorem~\ref{thm:advanced_composition}}
\label{app:composition}
Let us restate the theorem:
\composition*
\begin{proof}
The proof closely follows~\cite{abadi2016deep}.

First, we can write 
\begin{align}
	 L^{(1:T)} &= \log \frac{p(w^{(1)} \ldots w^{(T)}~|~D)}{p(w^{(1)} \ldots w^{(T)}~|~D')} \\
	 	&= \log \prod_{t=1}^T \frac{ p(w^{(t)}~|~w^{(t-1)} \ldots p(w^{(1)}, D)}{p(w^{(t)}~|~w^{(t-1)} \ldots p(w^{(1)}, D')} \\
	 	&= \log \prod_{t=1}^T \frac{ p(w^{(t)}~|~w^{(t-1)}, D) }{ p(w^{(t)}~|~w^{(t-1)}, D')} \\
	 	&= \sum_{t=1}^T L^{(t)}
\end{align}

Unlike the composition proof of the moments accountant by \citet{abadi2016deep}, we cannot simply swap the product and the expectation in our proof, because the additional example $x'$ remains the same in all applications of the privacy mechanism and probability distributions will not be independent. However, we can use generalised H\"older's inequality:
\begin{align}
\label{eq:holder}
	\left\| \prod_{t=1}^T f_t \right\|_r \leq \prod_{t=1}^T \| f_t \|_{p_t},
\end{align}
where $p_t$ are such that $\sum_{t=1}^T \frac{1}{p_t} = \frac{1}{r}$, and $\| f \|_r = \left( \int_S |f|^r dx \right)^{1/r}$. 

Choosing $r=1$ and $p_t = T$:
\begin{align}
	\mathbb{E}\left[e^{\lambda L^{1:T}}\right] 
		&= \mathbb{E}\left[ \prod_{t=1}^T e^{\lambda \log \frac{ p(w^{(t)}~|~w^{(t-1)}, X) }{ p(w^{(t)}~|~w^{(t-1)}, X')} } \right] \\
		&= \mathbb{E}_x\left[ \mathbb{E}_w \left[ \left. \prod_{t=1}^T e^{\lambda \log \frac{ p(w^{(t)}~|~w^{(t-1)}, X) }{ p(w^{(t)}~|~w^{(t-1)}, X')} } ~\right|~ x' \right] \right] \label{eq:expectation_split} \\
		&= \mathbb{E}_x\left[ \prod_{t=1}^T \mathbb{E}_w \left[ \left. e^{\lambda \log \frac{p(w^{(t)}~|~w^{(t-1)}, X)}{p(w^{(t)}~|~w^{(t-1)}, X')} } ~\right|~ x' \right] \right] \label{eq:noise_product_swap} \\
		&= \mathbb{E}_x\left[ \prod_{t=1}^T e^{\lambda \mathcal{D}_{\lambda+1} (p_t \| q_t)} \right] \\
		&\leq \prod_{t=1}^T \mathbb{E}_x \left[ e^{T\lambda \mathcal{D}_{\lambda+1} (p_t \| q_t)} \right]^{\frac{1}{T}}, \label{eq:data_product_swap}
\end{align}
where \eqref{eq:expectation_split} is by the law of total expectation; \eqref{eq:noise_product_swap} is due to independence of noise between iterations, similarly to~\cite{abadi2016deep}; and \eqref{eq:data_product_swap} is by H\"older's inequality.
\end{proof}

\stepcounter{theorem}

\subsection{Proof of Theorem~\ref{thm:gauss_privacy_cost}}
\label{app:gaussian}
Let us restate the theorem:
\gaussian*
\begin{proof}
Without loss of generality, assume $D' = D \cup \{x'\}$. For brevity, let $d_t = \|g_t - g_t'\|$.

Let us first consider $\mathcal{D}_{\lambda+1} (p(w|D') \| p(w|D))$:
\begin{align}
	\mathbb{E}&\left[\left(\frac{p(w|D')}{p(w|D)}\right)^{\lambda+1} \right]  \nonumber \\
		&= \mathbb{E}\left[\left(\frac{(1-q) \mathcal{N}(0, \sigma^2) + q \mathcal{N}(d_t, \sigma^2)}{\mathcal{N}(0, \sigma^2)}\right)^{\lambda+1} \right] \\
		&= \mathbb{E}\left[\left( (1-q) + q \frac{\mathcal{N}(d_t, \sigma^2)}{\mathcal{N}(0, \sigma^2)} \right)^{\lambda+1} \right] \\
		&= \mathbb{E}\left[\left( (1-q) + q e^\frac{(w - d_t)^2 - w^2}{2\sigma^2} \right)^{\lambda+1} \right] \\
		&= \mathbb{E}\left[\left( (1-q) + q e^\frac{2dw - d_t^2}{2\sigma^2} \right)^{\lambda+1} \right] \\
		&= \mathbb{E}\left[\sum_{k=0}^{\lambda+1} \binom{\lambda + 1}{k} q^k (1-q)^{\lambda+1-k} e^\frac{2d_t kw - kd_t^2}{2\sigma^2} \right] \label{eq:binomial} \\
		&= \sum_{k=0}^{\lambda+1} \binom{\lambda + 1}{k} q^k (1-q)^{\lambda+1-k} \mathbb{E}\left[ e^\frac{2d_t kw - kd_t^2}{2\sigma^2} \right] \label{eq:binomial_independence} \\
		&= \sum_{k=0}^{\lambda+1} \binom{\lambda + 1}{k} q^k (1-q)^{\lambda+1-k} e^{\frac{k^2 - k}{2\sigma^2} d_t^2} \label{eq:gaussian_expectation_property} \\
		&= \mathbb{E}_{k \sim B(\lambda+1, q)} \left[e^{\frac{k^2 - k}{2\sigma^2} \|g_t - g'_t\|^2} \right],
\end{align}
Here, in~\eqref{eq:binomial} we used the binomial expansion, in~\eqref{eq:binomial_independence} the fact that the factors in front of the exponent do not depend on $w$, and in~\eqref{eq:gaussian_expectation_property} the property $\mathbb{E}_w\left[\exp(2aw/(2\sigma^2))\right] = \exp(a^2/(2\sigma^2))$ for $w\sim \mathcal{N}(0,\sigma^2)$. Plugging the above in the privacy cost formula (Eq.~10 in the main paper), we get the expression for $c_t^{L}(\lambda)$.

Computing $\mathcal{D}_{\lambda+1} (p(w|D) \| p(w|D'))$ is a little more challenging. Let us first change to $\mathcal{D}_{\lambda} (p(w|D) \| p(w|D'))$, so that the expectation is taken over $\mathcal{N}(0, \sigma^2)$. Then, we can bound it observing that $f(x) = \frac{1}{x}$ is convex for $x > 0$ and using the definition of convexity, and apply the same steps as above:
\begin{align}
	\mathbb{E}&\left[\left(\frac{p(w|D)}{p(w|D')}\right)^{\lambda} \right]  \nonumber \\
		&= \mathbb{E}\left[\left(\frac{\mathcal{N}(0, \sigma^2)}{(1-q) \mathcal{N}(0, \sigma^2) + q \mathcal{N}(d_t, \sigma^2)}\right)^{\lambda} \right] \\
		&\leq \mathbb{E}\left[\left( (1-q) + q e^\frac{d_t^2 - 2dw}{2\sigma^2} \right)^{\lambda} \right] \\
		&= \mathbb{E}_{k \sim B(\lambda, q)} \left[e^{\frac{k^2 + k}{2\sigma^2} \|g_t - g'_t\|^2} \right]
\end{align}
In practice, we haven't found any instance of $\mathcal{D}_{\lambda+1} (p(w|D') \| p(w|D)) < \mathcal{D}_{\lambda+1} (p(w|D) \| p(w|D'))$ when the latter was computed using numerical integration, although it may happen when using this theoretical upper bound.
\end{proof}

\subsection{Proof of Theorem~\ref{thm:estimator}}
\label{app:estimator}
Let us restate the theorem:
\estimator*
\begin{proof}
First of all, we can drop the logarithm from our consideration because of its monotonicity.

Now, assuming that samples $e^{\lambda \hat{D}_{\lambda+1}^{(t)}}$ have a common mean and a common variance, and applying the maximum entropy principle in combination with an uninformative (flat) prior, one can show that the quantity $\frac{M(t) - \mathbb{E}\left[e^{\lambda \hat{D}_{\lambda+1}^{(t)}}\right]}{S(t)} \sqrt{m - 1}$ follows the Student's $t$-distribution with $m - 1$ degrees of freedom~\cite{oliphant2006bayesian}.

Finally, we use the inverse of the Student's $t$ CDF to find the value that this random variable would only exceed with probability $\gamma$. The result follows by simple arithmetical operations.
\end{proof}

\section{Evaluation}

\subsection{Experimental setting}
\label{app:setting}
All experiments were performed on a machine with Intel Xeon E5-2680 (v3), 256 GB of RAM, and two NVIDIA TITAN X graphics cards. We train a classifier represented by a neural network on MNIST~\cite{lecun1998gradient} and on CIFAR10~\cite{krizhevsky2009learning} using DP-SGD. The first dataset contains 60,000 training examples and 10,000 testing images. We use large batch sizes of $1024$, clip gradient norms to $C=1$, and $\sigma=0.1$. We also experimented with the idea of dropping updates for a random subset of weights, and achieved the best performance with updating $10\%$ of weights at each iteration. The second dataset consists of 50,000 training images and 10,000 testing images of objects split in 10 classes. For this dataset, we use the batch size of $512$, $C=1$, and $\sigma=0.8$. We fix $\delta=10^{-5}$ in all experiments, and $\delta_\mu=10^{-10}$ to achieve $(\varepsilon, 10^{-5})$ bound for $99.999\%$ of data distribution using Markov inequality. 

MNIST experiments are performed with the CNN model from Tensorflow tutorial (the same as in~\cite{abadi2016deep}, except we do not use PCA), trained using SGD with the learning rate 0.02. In case of CIFAR10, in order for our results to be comparable to~\citep{abadi2016deep}, we pre-train convolutional layers of the model on a different dataset and retrain a fully-connected layer in a privacy-preserving way. We were unable to reproduce the experiment exactly as specified in~\citep{abadi2016deep} and chose a different model (VGG-16 pre-trained on ImageNet), guided by maintaining a similar or lower non-private accuracy. The model was trained using Adam with the learning rate of 0.001. Since the goal of these experiments is to show relative performance of private methods, we did not perform an exhaustive search for hyperparameters, either using default or previously published values or values that yield reasonable training behaviour.

Privacy accounting with DP-SGD works in the following way. The non-private learning outcome at each iteration $t$ is the gradient $g_t$ of the loss function w.r.t. the model parameters, the outcome distribution is the Gaussian $\mathcal{N}(g_t, \sigma^2 C^2)$. Before adding noise, the norm of the gradients is clipped to $C$. For the moments accountant, the privacy loss is calculated using this $C$ and $\sigma$. For the Bayesian accountant, either pairs of examples $x_i, x_j$ or pairs of batches are sampled from the dataset at each iteration, and used to compute $\hat{c}_t(\lambda)$. Although clipping gradients is no longer necessary with the Bayesian accountant, it is highly beneficial for incurring lower privacy loss at each iteration and obtaining tighter composition. Moreover, it ensures the classic DP bounds on top of BDP bounds.

We also run evaluation on two binary classification tasks taken from UCI database: Abalone~\cite{waugh1995extending} (predicting the age of abalone from physical measurements) and Adult~\cite{kohavi1996scaling} (predicting income based on a person's attributes). In this setting, we compare differentially private variational inference (DPVI-MA~\cite{jalko2016differentially}) to the variational inference with BDP. The datasets have 4,177 and 48,842 examples with 8 and 14 attributes accordingly. We use the same pre-processing and models as~\cite{jalko2016differentially}. We run experiments using the authors original implementation\footnote{\url{https://github.com/DPBayes/DPVI-code}} with slight modifications (e.g. accounting randomness of sampling from variational distributions, instead of adding noise, using Bayesian accountant, and performing classification with variational samples instead of optimal variational parameters).

\subsection{Effect of $\sigma$ and bounded sensitivity}
\label{sec:sigma}

\begin{figure*}
	\centering
	\begin{subfigure}{0.3\textwidth}
    		\includegraphics[width=\textwidth]{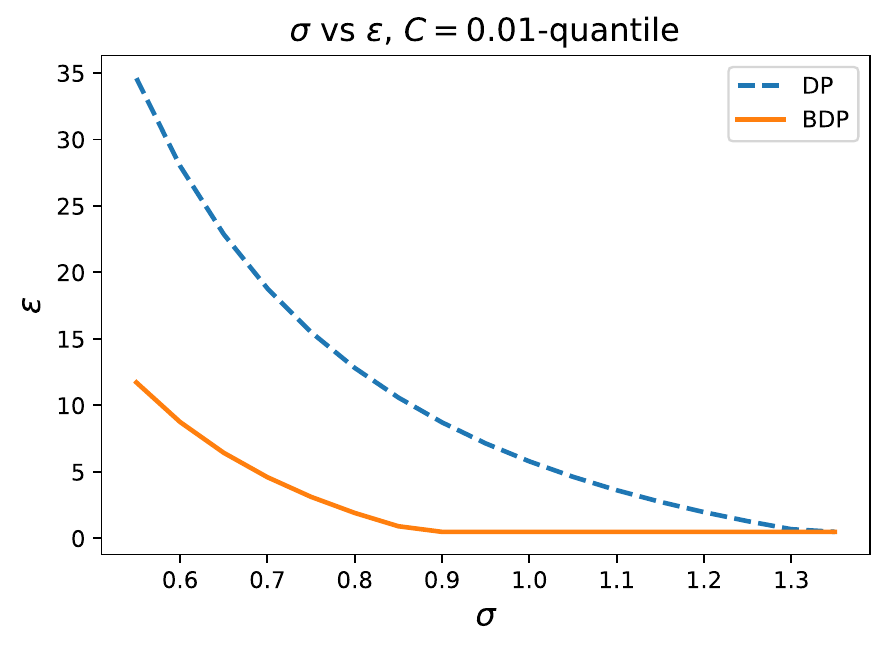}
    		\caption{Clipping at $0.01$-quantile of $\|\nabla f\|$.}
    		\label{fig:eps_sigma_with_clipping_a}
	\end{subfigure}
	\begin{subfigure}{0.3\textwidth}
    		\includegraphics[width=\textwidth]{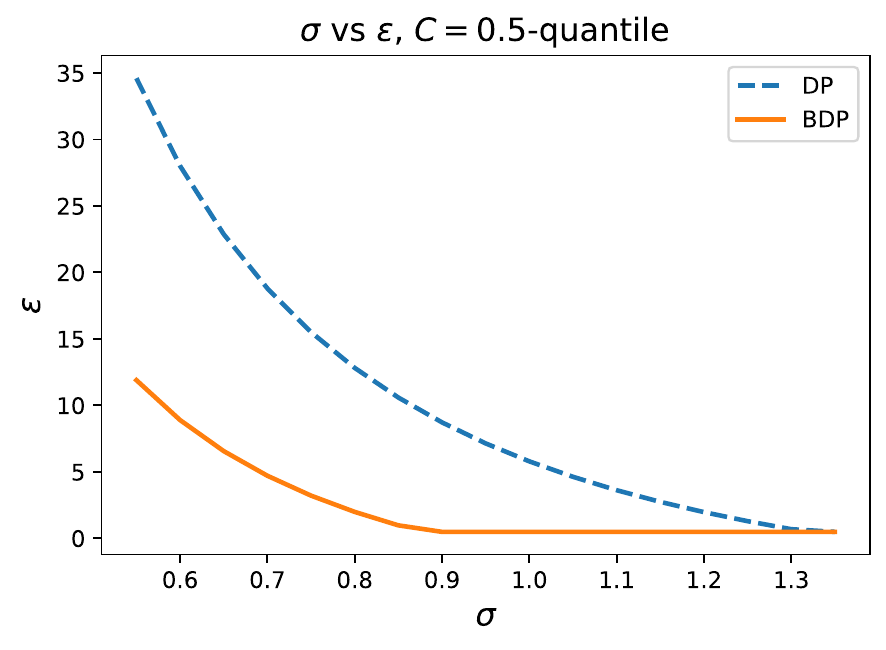}
    		\caption{Clipping at $0.50$-quantile of $\|\nabla f\|$.}
	\end{subfigure}
	\begin{subfigure}{0.3\textwidth}
    		\includegraphics[width=\textwidth]{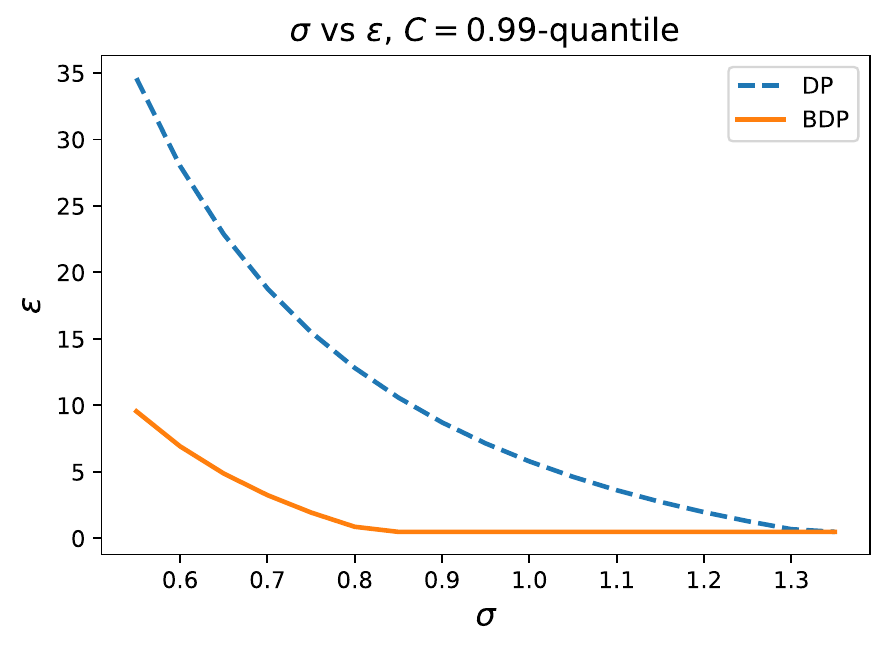}
    		\caption{Clipping at $0.99$-quantile of $\|\nabla f\|$.}
	\end{subfigure}
	\caption{Dependency between $\sigma$ and $\varepsilon$ for different $C$ when clipping for both DP and BDP.}
	\label{fig:eps_sigma_with_clipping}
\end{figure*}

\begin{figure*}
	\centering
	\begin{subfigure}{0.3\textwidth}
    		\includegraphics[width=\textwidth]{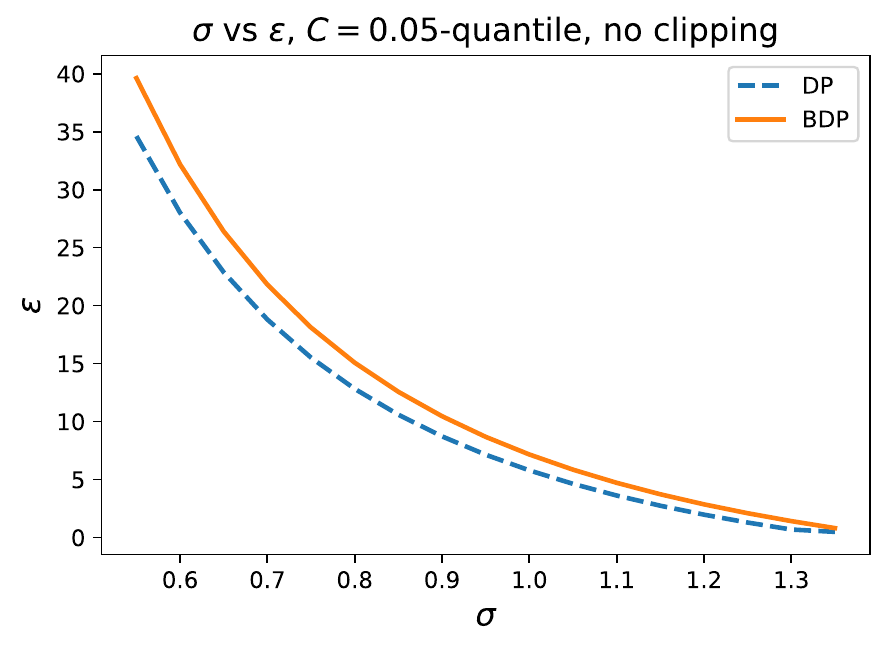}
    		\caption{Noise at $0.05$-quantile of $\|\nabla f\|$.}
	\end{subfigure}
	\begin{subfigure}{0.3\textwidth}
    		\includegraphics[width=\textwidth]{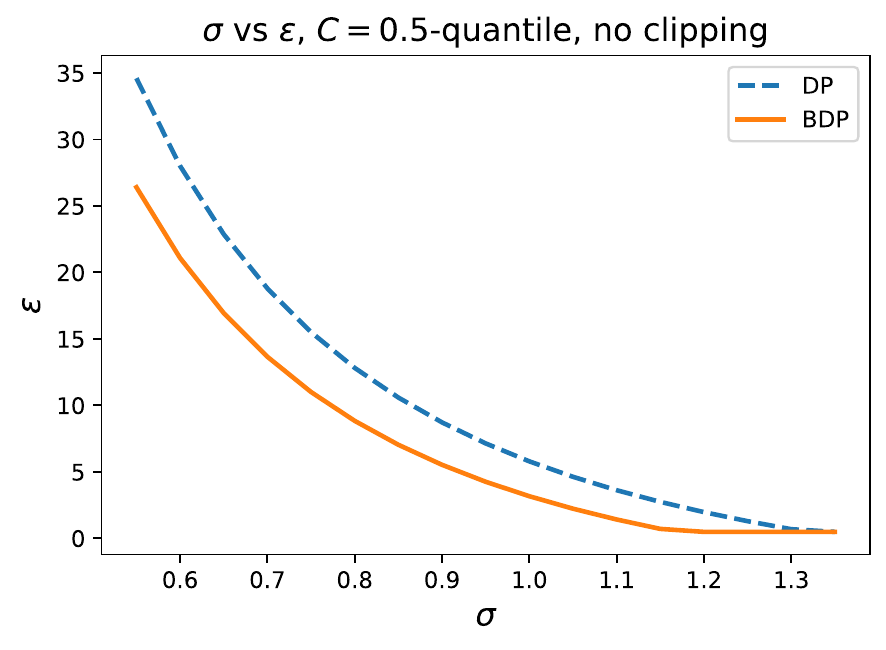}
    		\caption{Noise at $0.50$-quantile of $\|\nabla f\|$.}
	\end{subfigure}
	\begin{subfigure}{0.3\textwidth}
    		\includegraphics[width=\textwidth]{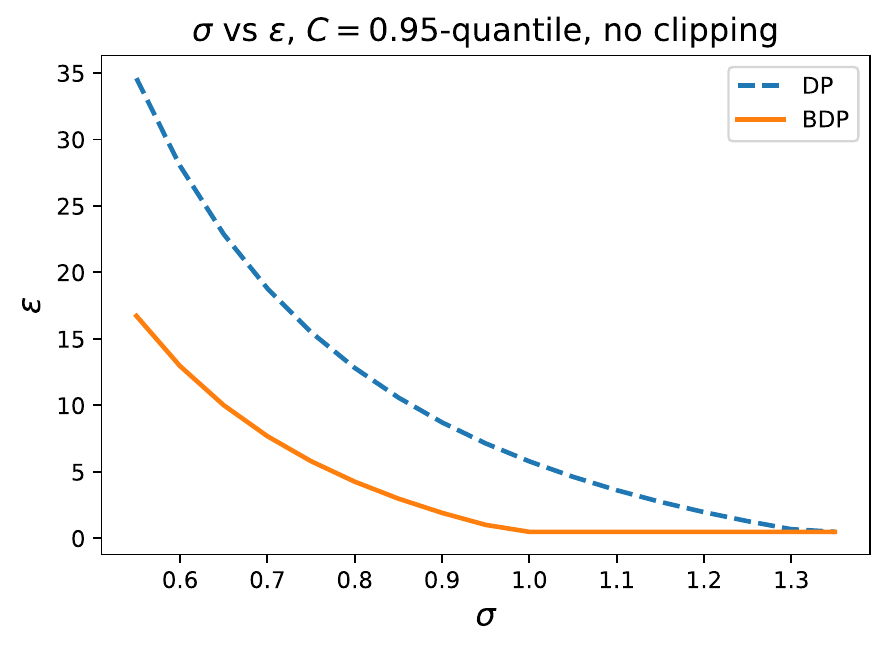}
    		\caption{Noise at $0.95$-quantile of $\|\nabla f\|$.}
	\end{subfigure}
	\caption{Dependency between $\sigma$ and $\varepsilon$ for different $C$ when clipping for DP and not clipping for BDP.}
	\label{fig:eps_sigma_no_clipping}
\end{figure*}

The primary goal of our paper is to obtain more meaningful privacy guarantees sacrificing as little utility as possible. The main factor in the loss of utility is the variance of the noise we add during training. Therefore it is critical to examine how our guarantee behaves compared to the classic DP for the same amount of noise. Or equivalently, how much noise does it require to reach the same $\varepsilon$.

As stated above, there are two possible regimes of operation for the Gaussian noise mechanism under Bayesian differential privacy: with bounded sensitivity and with unbounded sensitivity. The first is just like the classic DP: there is a maximum bound on the contribution of an individual example, and the noise is scaled to it. The second does not have a bound on contribution and mitigates it by taking into account the low probability of extreme contributions.  

Figures~\ref{fig:eps_sigma_with_clipping} and~\ref{fig:eps_sigma_no_clipping} demonstrate the dependency between $\sigma$ and $\varepsilon$ for different clipping thresholds $C$ chosen relative to the quantiles of the gradient norm distribution. If we bound sensitivity by clipping the gradients, it ensures that BDP always requires less noise than DP to reach the same $\varepsilon$, as seen in Figure~\ref{fig:eps_sigma_with_clipping}. As we decrease the clipping threshold $C$, more and more gradients get clipped and the BDP curve approaches the DP curve (Figure~\ref{fig:eps_sigma_with_clipping_a}). However, as we observe in Figure~\ref{fig:eps_sigma_no_clipping} comparing DP with bounded sensitivity and BDP with unbounded sensitivity, using unclipped gradients results in less consistent behaviour. It may require a more thorough search for the right noise variance to reach the same $\varepsilon$.

\subsection{Effect of $\lambda$}
\label{sec:lambda}

As mentioned in Section~4.2, the privacy cost, and therefore the final value of $\varepsilon$, depend on the choice of $\lambda$. We run the Bayesian accountant for the Gaussian mechanism with the fixed pairwise gradient distances (s.t. these results apply exactly to the moments accountant) for different signal-to-noise ratios and different $\lambda$.

Depicted in Figure~\ref{fig:lambda_eps_c} is $\varepsilon$ as a function of $\lambda$ for 10000 steps. We observe that $\lambda$ has a clear effect on the final $\varepsilon$ value. In some cases this effect is very significant and the change is sharp. It suggests that in practice one should be careful about the choice of $\lambda$. We also note that for lower signal-to-noise ratios (e.g. $C=0.1, \sigma=1$) the optimal choice of $\lambda$ is much further on the real line and may well be outside the typically range computed in the literature.

\begin{figure}
\centering
\includegraphics[width=\linewidth]{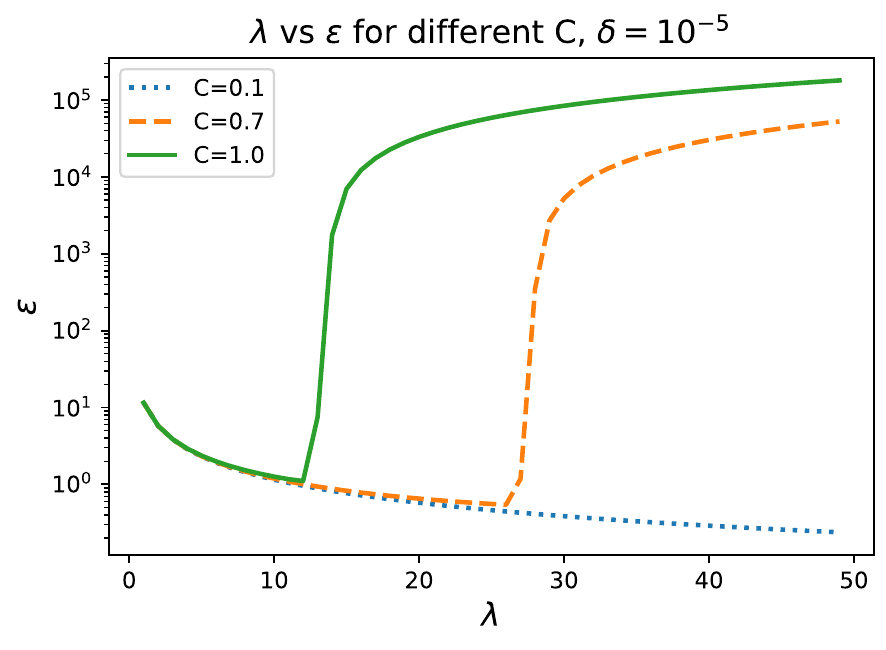}
\caption{Dependency of $\lambda$ and $\varepsilon$ for different clipping thresholds $C$, $q = 64 / 60000$, $\sigma=1.0$.}
\label{fig:lambda_eps_c}
\end{figure}

\end{document}